\documentclass[11pt]{article}
\usepackage[a4paper,  total={6in, 8in}]{geometry}

\usepackage{amsmath,amssymb,amsthm}
\usepackage{algorithm,algorithmic}
\usepackage{dsfont}

\newtheorem{lemma}{Lemma}
\newtheorem{theorem}{Theorem}

\newtheorem{definition}{Definition}
\newtheorem{condition}{Condition}
\newtheorem{proposition}{Proposition}

\usepackage{algorithm}
\usepackage{algorithmic}

\usepackage{graphicx} 
\usepackage{subfigure} 
\usepackage{url}

\newcommand{\Norm}[1]{ \| #1 \| }

\title{Outlier Robust Online  Learning}
\date{}
\author{
	Jiashi~Feng
	\\
	Department of Electrical and Computer Engineering\\
	National University of Singapore\\
	\texttt{elefjia@nus.edu.sg} \\ \\
	Huan~Xu \\
	School of Industrial and Systems Engineering \\
	Georgia Institute of Technology \\ 
	\texttt{huan.xu@isye.gatech.edu} \\\\
	Shie~Mannor\\
	Department of Electrical Engineering\\
	Technion~--~Israel Institute of Technology\\
	\texttt{shie@ee.technion.ac.il} 
}
\begin{document}
\maketitle

\begin{abstract}
	We consider  the  problem of learning from noisy data in practical settings where the size of data is too large to store on a single machine. More challenging, the data coming from the wild may contain malicious outliers. To address the scalability and robustness issues, we present an \emph{online} robust  learning (ORL)  approach. ORL is simple to implement and has provable  robustness guarantee\textemdash in stark contrast to existing online learning approaches that are generally fragile to outliers. We  specialize the  ORL approach for two concrete cases: online robust principal component analysis and online  linear regression. We demonstrate the efficiency and  robustness advantages of  ORL through comprehensive simulations and predicting image tags on a large-scale data set. We also discuss  extension of the ORL to distributed learning and provide experimental evaluations.
\end{abstract}

\section{Introduction}
\label{sec:introduction}

%
%
%
%
In the  era of big data, traditional statistical learning  methods are facing two  significant challenges: (1) how to scale   current machine learning methods to the large-scale data? And (2) how to obtain accurate inference results when the data are noisy and may even contain malicious outliers? These two important challenges naturally lead to a need for  developing   \emph{scalable robust} learning methods.



Traditional robust learning methods generally rely on optimizing certain robust statistics~\cite{maronna1998robust,HR-PCA} or applying some sample trimming strategies~\cite{donoho1992breakdown}, whose calculations require loading all the samples into the memory or going through the data multiple times~\cite{feng2013stochastic}. Thus, the computational time of those robust learning methods is usually at least linearly dependent on the size of the sample set, $N$. 
For example, in RPCA~\cite{HR-PCA}, the computational time is $O(Np^2r)$ where $r$ is the intrinsic dimension of the subspace and $ p $ is the ambient dimension. In robust linear regression~\cite{chen2013noisy}, the computational time is super-linear on the sample size: $O(pN\log N)$. This rapidly increasing computation time becomes a major obstacle for applying robust learning methods to big data in practice, where the sample size easily reaches the terabyte or even petabyte scale.

Online learning and distributed learning are natural solutions to the scalability issue. Most of existing online statistical learning methods propose to optimize a surrogate function in an online fashion, such as employing stochastic gradient descent \cite{guan2012online,mairal2009online,feng2013online} to update the estimates, which however cannot handle the outlier samples in the streaming data \cite{huber2011robust}.   
Similarly, most of existing distributed learning approaches (\emph{e.g.}, MapReduce~\cite{dean2008mapreduce}) are not robust to  contamination from outliers, communication errors or computation node breakdown. 


In this work, we propose an online robust learning (ORL)  framework to efficiently process big data with outliers while preserving robustness and statistical consistency of the estimates. The core technique is based on two-level online learning procedure, one of which employs a novel  median filtering process. The robustness of median has been  investigated in statistical estimations for heavy-tailed distributions \cite{minsker2013geometric,hsu2013loss}. However, to our best knowledge, this work is among the first to  employ such estimator to deal with outlier samples in the context of online learning.

The implementation of ORL follows mini-batch based online optimization  which is popular in a wide range of machine learning problems (\emph{e.g.}, deep learning, large-scale SVM) from large-scale data.  Within each mini-batch, ORL computes an independent  estimate. However,   outliers may be heterogeneously distributed on the mini-batches and some of them may contain overwhelmingly many outliers. The corresponding estimate will be arbitrarily bad and break down the overall online learning. Therefore, on top of such streaming estimates  ORL performs another level of robust estimation\textemdash median filtering\textemdash to obtain reliable estimate.  The ORL approach is general and compatible with many  popular  learning algorithms.
Besides its obvious advantage of enhancing the computation efficiency for handling big data,  ORL incurs negligible robustness loss compared to  centralized (and computationally unaffordable) robust learning methods. In fact, we provide  analysis and demonstrate that ORL is robust to a constant fraction of ``bad'' estimates generated in the streaming mini-batches that are corrupted by outliers.

We specify  the  ORL approach for two concrete problems\textemdash online robust principal component analysis (PCA) and  linear regression. Comprehensive experiments on both synthetic and real large scale datasets  demonstrate the efficiency and robustness advantages of the proposed ORL approach. In addition, ORL can be adapted straightforwardly to distributed learning setting and offers additional robustness to corruption of several computation nodes or communication errors, as demonstrated in the experiments.

In short, we make following contributions in this work. First, we develop an outlier robust online learning framework which is the first one with provable robustness to a constant fraction of outliers. Secondly, we introduce two concrete online robust learning approaches, one for unsupervised learning and the other for supervised learning. Other examples can be developed in a similar way easily.  Finally, we also present the  application of the ORL approach to distributed learning setting which is equally attractive for learning from large scale data. 

\section{Preliminaries}

\subsection{Problem Set-up}	
We consider a set of $N{=}n{+}n_1$ observation samples $\mathcal{X} = \mathcal{X}_\mathcal{I} \cup \mathcal{X}_\mathcal{O} = \{\mathbf{x}_1,\ldots,\mathbf{x}_n\} \cup \{\mathbf{x}_{n+1},\ldots, \mathbf{x}_{n+n_1}\} \subset \mathbb{R}^{p}$, which contains a mixture of $n$ authentic samples $\mathcal{X}_\mathcal{I}$ and $n_1$ outliers $ \mathcal{X}_\mathcal{O} $.  The authentic samples are generated according to an underlying model (\emph{i.e.}, the ground truth)  parameterized by ${\theta}^\star \in \Theta$.  The target of a statistical learning procedure is to estimate the model parameter $ \theta^\star $ according to the provided observations $ \mathcal{X} $. Throughout the paper, we assume the authentic samples are sub-Gaussian random vectors in $ \mathbb{R}^p $, which thus satisfy that
\begin{equation}\label{eqn:subgaussian}
\mathbb{P}(|\langle \mathbf{x},\mathbf{u} \rangle |>t)\leq 2e^{-t^2/L^2}  \text{ for } t>0  \text{ and } \mathbf{u} \in {S}^{p-1},
\end{equation}
for some $ L $. Here ${S}^{p-1}$ denotes the unit sphere. 


In this work, we focus on the  case where a constant fraction of the observations are outliers, and  we use $\lambda \triangleq {n_1}/N$ to denote this  outlier fraction. 
In the context of online learning, samples are provided in a sequence of $ T $ mini batches, each of which contains $ b = \lfloor (n+n_1)/T \rfloor $ observations. Denote the sequence as $  \{\mathcal{X}_1,\ldots,\mathcal{X}_T\} = \{\mathbf{x}_1,\ldots,\mathbf{x}_b, \mathbf{x}_{b+1},\ldots, \mathbf{x}_{Tb}\} $. The target of online statistical learning is to estimate the parameter $ \theta^\star $, only based on the observations revealed so far. 


\subsection{Geometric Median}
We first introduce the \emph{geometric median} here\textemdash a core concept underlying the median filtering procedure that is important for developing the proposed online robust learning approach.



\begin{definition}[Geometric Median]
	\label{def:median}
	Given a finite collection of i.i.d.\ estimates $\theta_1,\ldots,\theta_T  \in \Theta$, their geometric median is the point which minimizes the total $\ell_1$ distance to all the given estimates, \emph{i.e.},
	\begin{equation}
	\label{eqn:median}
	\widehat{\theta} = \mathrm{median}(\theta_1,\ldots,\theta_T) := \underset{\theta \in \Theta}{\arg\min}\sum_{j=1}^T\|\theta -\theta_j\|.
	\end{equation}
\end{definition}

An important property of the geometric median is that it indeed aggregates a collection of independent estimates into a single estimate $\widehat{\theta}$ with strong concentration guarantees, even in presence of a constant fraction of outlying estimates in the collection.
The following lemma, straightforwardly derived from Lemma 2.1 in~\cite{minsker2013geometric}, characterizes such robustness property of the geometric median.
\begin{lemma}
	\label{lemma:median}
	Let $\widehat{\theta}$ be the geometric median of the points $\theta_1,\ldots,\theta_T \in \Theta$. Fix $\gamma \in \left(0,\frac{1}{2}\right)$ and $C_\gamma = (1-\gamma)\sqrt{\frac{1}{1-2\gamma}}$. Suppose
	there exists a subset $J\subseteq \{1,\ldots,T\}$ of cardinality $|J|>(1-\gamma) T$ such that for all $j\in J$ and any point $\theta^\star\in \Theta$, $\|\theta_j -\theta^\star\| \leq r$. Then we have $\|\widehat{\theta} - \theta^\star\| \leq C_\gamma r$.
\end{lemma}
In words, given a set of points, their geometric median will be close to the ``true'' $\theta^\star$ as long as at least half of them are close to  $\theta^\star$. In particular, the geometric median will not be skewed severely even if some of the points deviate significantly away from  $\theta^\star$.

\section{Online Robust Learning}
In this section, we present how to scale up robust learning algorithms to process large-scale data (containing outliers) through online learning without losing robustness. We term the proposed approach as online robust learning (ORL).

The idea behind ORL is intuitive\textemdash instead of equally incorporating generated estimates at each time step, ORL aggregates the sequentially generated estimates by mini-batch based learning methods via an  \emph{online} computation of the robust geometric median. Basically,  ORL runs online learning at two levels: at the bottom level, ORL employs appropriate robust learning procedures $ \operatorname{RL}(\cdot, \nu) $ with parameter $\nu$ (\emph{e.g.}, robust PCA algorithms on a mini-batch of samples) to obtain a sequence of estimates $ \{\theta_1,\ldots,\theta_T\} $ of $ \theta^\star $ based on the observation mini-batch $ \mathcal{X}_1,\ldots,\mathcal{X}_T $; at the top  level, ORL updates the running estimate $ \widehat{\theta}_t $ ($1\leq t \leq T$)  through a   geometric median  filtering algorithm (explained later) over $\widehat{\theta}_1, \ldots, \widehat{\theta}_{t-1}$ and outputs a \emph{robust}  estimate after going through all the mini-batches. Intuitively, according to Lemma \ref{lemma:median}, as long as a majority of mini-batch estimates are not skewed by outliers, the produced $\widehat{\theta}_t$ would be robust and accurate. This new two-level robust learning gives ORL stronger robustness to outliers compared with ordinary  online learning.

To  develop  the top level geometric median filtering procedure, recall  definition of the geometric median in~\eqref{eqn:median}.
A natural estimate of the geometric median $\widehat{\theta}$ is the minimum $ \widehat{\theta}_T $ of the following empirical loss function $\widehat{G}_T$:
\begin{equation}
\label{eqn:emp_median}
\widehat{\theta}_T = \underset{\theta \in \Theta}{\arg\min}  \left\{\widehat{G}_T  \triangleq    \frac{1}{T}\sum_{i=1}^T \|\theta_i-\theta\|\right\}.
\end{equation}
The empirical function  $ \widehat{G}_T $ is differentiable everywhere except for the points $ \theta_i $,
and can be optimized  by applying  stochastic gradient descent (SGD) \cite{bottou1998online}. More concretely, at the time step $ t $, given a new estimate $\theta_{t+1}$ (based on the $(t{+}1)$-st mini-batch)  and the current  estimate $ {\theta}_t $, ORL computes the gradient at point $ \theta $ of the empirical  function $ \widehat{G}_T $  in Eqn.~\eqref{eqn:emp_median} evaluated only at $\theta_{t+1}$:
\begin{equation}
\label{eqn:sgdmed}
\widehat{g}(\theta;\theta_{t+1}) \triangleq  \frac{\partial \widehat{G}_T (\theta;\theta_{t+1})}{\partial \theta} = \frac{2(\theta-\theta_{t+1})}{\|\theta-\theta_{t+1}\|}.
\end{equation}
Then  ORL updates the   estimate $ \widehat{\theta}_t $  by following filtering: 
\begin{equation}
\label{eqn:online-median}
\widehat{\theta}_{t+1} \leftarrow \widehat{\theta}_{t} -  \eta_t \widehat{g}(\widehat{\theta}_{t};\theta_{t+1}) = (1-w_t) \widehat{\theta}_t + w_t {\theta}_{t+1}.
\end{equation}
Here $ \eta_t $ is a predefined step size parameter which usually takes the form of $ 1/c_a t $ with a constant $ c_a $ characterizing  convexity of the empirical function to optimize. Besides, $w_t = 2\eta_t/\|\widehat{\theta}_t - \theta_{t+1}\|$ controls contribution of each new estimate $\theta_{t+1}$  conservatively in  updating the global estimate $\widehat{\theta}$.
Details of ORL are provided in Algorithm~\ref{alg:ROL}.
\begin{algorithm}[t]
	\caption{The ORL Approach}
	\label{alg:ROL}
	\begin{algorithmic}
		\STATE \textbf{Input}: 
		 Mini-batch sequence $\mathcal{X}_1,\ldots,\mathcal{X}_T$,
		convexity parameter $ c_a $,  robust learning procedure parameter $\nu$.
		\STATE \textbf{Initialization}: $ \widehat{\theta}_0= 0  $.
		\FOR{$t=1,\ldots,T$}
		\STATE Call the robust learning procedure $\theta_t = \mathrm{RL}(\mathcal{X}_t, \nu)$;
		\STATE Compute weight $w_t = 2\eta_t/\|\widehat{\theta}_{t-1} - \theta_{t}\|$ with $\eta_t = 1/c_a t$.
		\STATE Update the estimate: $\widehat{\theta}_{t}=(1-w_t)\widehat{\theta}_{t-1} + w_t \theta_{t+1}$.
		\ENDFOR
		\STATE \textbf{Output}:  Final estimate $\widehat{\theta}_T$. 
	\end{algorithmic}
\end{algorithm}
Another level of filtering is important. Certain mini-batches may contain overwhelming outliers. Therefore, even though a robust learning procedure is employed on each mini-batch, the resulted estimate cannot be guaranteed to be accurate. In fact, a mini-batch containing over $50\%$ outliers would corrupt any robust learning procedure\textemdash the resulted estimate can be arbitrarily bad and breakdown the overall online learning. To address this critical issue,  ORL performs another level of online learning for updating the ``global'' estimate with adaptive weights for the new estimate and ``filters out'' possibly bad estimates.

\section{Performance Guarantees}
We provide the performance guarantees for  ORL in this section. 
Throughout this section, we use following asymptotic inequality notations: for positive numbers $ a $ and $ b $, the asymptotic inequality $ a \lesssim_{p,q} b $ means that $ a \leq C_{p,q}b $ where $C_{p,q}$ is a constant only depending on $p,q$.
	Suppose $ N $ samples, a constant fraction of which are authentic ones and have sub-Gaussian distributions as specified in \eqref{eqn:subgaussian} for some $ L $, are evenly divided to $ T $ mini-batches  and  outlier fractions on the $ T $ mini-batches are $\lambda_1,\ldots,\lambda_T $ respectively. Let $ {\theta}_1, \ldots, {\theta}_T $ be a collection of independent estimates of $ \theta^\star $ output by implementing the robust learning procedure $\mathrm{RL}(\cdot,\nu)$ on the $ T $  mini-batches independently. We assume an estimate or the robust learning procedure provides following composite deviation bound,
	\begin{equation}
	\label{eqn:bound_base_alg}
	\mathbb{P}\left(\Norm{{\theta}_i - \theta^\star}  \lesssim_{\delta,L} \sqrt{\frac{1}{b}} + \frac{\lambda_i}{1-\lambda_i}\sqrt{p}   \right) \geq 1-\delta,
	\end{equation}	
	where $b$ is the size of each mini-batch whose value can be tuned by the desired accuracy (\emph{e.g.}, through data augmentation). We will specify value of the constant depending on $\delta$ and $L$ explicitly in concrete applications. The above bound indicates the estimation error depends on the standard statistically error and the outlier fraction.  If $\lambda_i$ is overwhelmingly large, the estimate will be arbitrarily bad.

We now proceed to demonstrate that the ORL approach is robust to outliers\textemdash even on a constant fraction of mini-batches, the obtained estimates are not good, ORL can still provide reliable estimate with bounded error.  Given a sequence of estimates $\theta_1,\ldots,\theta_T$ produced internally in ORL,  we analyze and provide guarantee on performance of the ORL through following two steps. We first demonstrate the geometric median function $ G_T(\theta) $ is in fact strongly convex and thus geometric median filtering  provides a good estimate of the ``true'' geometric median of $\theta_1, \ldots, \theta_T$. Then we  derive following performance guarantee for ORL by invoking the robustness property of geometric median.
\begin{proposition}
	\label{prop:rol}
	Suppose  in total $ N $ samples, a constant fraction of  which have sub-Gaussian distribution as in \eqref{eqn:subgaussian}, are divided into $ T $ sequential mini batches of size $b$ with outlier fraction $\lambda_1,\ldots,\lambda_T$. Here $ T \geq 4 $.   We run a  base robust learning algorithm having a deviation bound as in \eqref{eqn:bound_base_alg} on each mini batch. Denote the ground truth of the parameter to estimate as $ \theta^\star $ and the output of ORL (Alg.~\ref{alg:ROL}) as $ \widehat{\theta}_T $. Then with probability at least $1-\delta$, $ \widehat{\theta}_T $ satisfies:
	\begin{equation*}
	\|\widehat{\theta}_T - \theta^\star\| \lesssim_{\delta,L, p, \gamma}  \frac{ \log(\log(T))+1}{T} +  \sqrt{\frac{1}{b}} +  \lambda(\gamma)\sqrt{p}.
	\end{equation*}
	Here $\lambda(\gamma)=\lambda_{(1-\gamma)}/(1-\lambda_{(1-\gamma)})$  and  $\lambda_{(1-\gamma)}$ denotes the $ \lfloor (1-\gamma) T \rfloor $ smallest outlier fraction in $ \{\lambda_1,\ldots,\lambda_T\} $ with $ \gamma \in [0,1/2) $.
\end{proposition}
The above results demonstrate the estimation error  of ORL consists of two components. The first term accounts for the deviation between the solution $ \widehat{\theta}_T $ and the ``true'' geometric median of the $T$ sequential estimations. When $ T$ is sufficiently large, \emph{i.e.}, after ORL seeing sufficiently many mini batches of observations, this error vanishes at a rate of $ O(\log\log(T)/T) $. The second term explains the deviation of geometric median of estimates from the ground truth. The significant part of this result is that the error of ORL only depends on the $\lfloor (1-\gamma) T \rfloor$ smallest outlier fraction among $ T $ mini-batches, no matter how severely the other estimates are corrupted. This explains why ORL is \emph{robust} to outliers in the samples.

\section{Application Examples}
In this section, we provide two concrete examples of the ORL approach, including one unsupervised learning algorithm\textemdash principal component analysis (PCA) and one supervised learning one\textemdash  linear regression (LR). Both  algorithms are  popular in practice but their online learning versions with robustness guarantees are still absent. Finally, we also discuss an extension of ORL for distributed robust learning. 
\subsection{Online Robust PCA}

Classical PCA is known to be fragile to outliers and many robust PCA methods have been proposed so far (see~\cite{HR-PCA} and  references therein). However, most of those methods require to load all the data into memory and have computational cost (super-)linear in the sample size, which prevents them from being applicable for big data. In this section, we first develop a new robust PCA method which robustifies PCA via  a robust sample covariance matrix estimation, and then demonstrate how to implement it with the ORL approach to enhance the efficiency.

Given a sample matrix $X=[\mathbf{x}_1,\mathbf{x}_2,\ldots,\mathbf{x}_n] \in \mathbb{R}^{p \times n}$, the standard covariance matrix is computed as $C=XX^\top$, \emph{i.e.}, $C_{ij} = \langle X_i, X_j\rangle,\forall i,j=1,\ldots,p$. Here $X_i$ denotes the $i$th row vector of matrix $X$. To obtain a robust estimate of the covariance matrix, we replace the vector inner product by a trimmed inner product, $\widehat{C}_{ij} = \langle X_i, X_j\rangle_{n_1}$, as proposed in \cite{chen2013robust} for linear regressor estimation. Intuitively, the trimmed inner product removes the outliers having large magnitude and the remaining outliers are bounded by inliers. Thus, the obtained covariance matrix, after proper symmetrization, is close to the authentic sample covariance. How to calculate the trimmed inner product  for a robust estimation of sample covariance matrix is given in Algorithm \ref{alg:trimmed_ip}. 

\begin{algorithm}[h]
	\caption{Trimmed inner product $\langle\mathbf{x},\mathbf{x}' \rangle_{n_1}$}
	\label{alg:trimmed_ip}
	\begin{algorithmic}
		\STATE {\bfseries Input:} Two vectors $\mathbf{x} \in \mathbb{R}^N$ and $\mathbf{x}' \in \mathbb{R}^N$, trimmed parameter $n_1$.
		\STATE Compute $q_i = \mathbf{x}_i\mathbf{x}'_i, i=1,\ldots,N$.
		\STATE Sort $\{|q_i|\}$ in ascending order and select the smallest $(N-n_1)$ ones.
		\STATE Let $\Omega$ be the set of selected indices.
		\STATE {\bfseries Output:} $h=\sum_{i\in\Omega}q_i$.
	\end{algorithmic}
\end{algorithm}

Then we perform a standard eigenvector decomposition on the covariance matrix to produce the principal component estimations. 
The details of the new Robust Covariance PCA (RC-PCA) algorithm  are provided in Algorithm \ref{alg:rpca}.

\begin{algorithm}[h]
	\caption{Robust Covariance PCA (RC-PCA)}
	\label{alg:rpca}
	\begin{algorithmic}
		\STATE {\bfseries Input:} Sample matrix $X = [\mathbf{x}_1,\ldots,\mathbf{x}_{N}] \in \mathbb{R}^{p\times N}$, subspace dimension $d$, outlier fraction $\lambda$.
		\STATE Compute the trimmed covariance matrix $\widehat{\Sigma}$:
		$\widehat{\Sigma}_{ij} = \langle X^i,X^j \rangle_{\lambda N},\forall i,j=1,\ldots,p$.
		\STATE Perform eigen decomposition on $\widehat{\Sigma}^\prime = (\widehat{\Sigma} + \widehat{\Sigma}^\top)/2$ and take the eigenvectors corresponding to the  $d$ largest  eigenvalues $\widehat{P}_\mathcal{U} = [\widehat{\mathbf{w}}_1,\ldots,\widehat{\mathbf{w}}_d]$.
		\STATE {\bfseries Output:} column subspace projector $\widehat{P}_\mathcal{U}$.
	\end{algorithmic}
\end{algorithm}

Applying the proposed ORL approach onto the above RC-PCA  develops a new online robust PCA algorithm, called ORL-PCA, as explained  in Algorithm \ref{alg:rol_pca}.
\begin{algorithm}[h!]
	\caption{ORL-PCA}
	\label{alg:rol_pca}
	\begin{algorithmic}
		\STATE \textbf{Input}: Sequential mini-batches $\mathcal{X}_1,\ldots,\mathcal{X}_T $ with size $b$, subspace dimension $ d $, {\sc RC-PCA} parameter $ q = 0.5 b $.
		\STATE \textbf{Initialization}: $\widehat{\Sigma}^{(0)} = 0 \in \mathbb{R}^{p \times p}$.
		\FOR{$t=1,\ldots,T$}
		\STATE Perform RC-PCA on $ \mathcal{X}_t $: $ P_\mathcal{U}^{(t)} = {\text{RC-PCA}}(\mathcal{X}_t;d,q) $;
		\STATE Compute covariance matrix: $ \Sigma^{(t)} = P_\mathcal{U}^{(t)} {P_\mathcal{U}^{(t)}}^\top $;
		\STATE Compute $w_t = 2\eta_t/{\|\widehat{\Sigma}^{(t-1)} - \Sigma^{(t)}\|}$ with $\eta_t = 1/c_a t$;
		\STATE Update the estimate $\widehat{\Sigma}^{(t)} = (1-w_t)\widehat{\Sigma}^{(t-1)} + w_t \Sigma^{(t)}$.
		\ENDFOR
		\STATE \textbf{Output}:  $ \widehat{P}_\mathcal{U}^{(T)} =$ \sc{svd}$\left(\widehat{\Sigma}^{(T)}, d\right)$.
	\end{algorithmic}
\end{algorithm}

Based on the above result, along with Proposition \ref{prop:rol}, we provide the following performance guarantee for ORL-PCA.
\begin{theorem}
	\label{theo:rol_pca}
	Suppose samples are divided into $ T $ mini-batches of size $b$. Authentic samples satisfy the sub-Gaussian distribution with parameter $L$. Let $ \lambda(\gamma) = \lambda_{(1-\gamma)}/(1-\lambda_{(1-\gamma)})$ where $\lambda_{(1-\gamma)}$ is the $(1-\gamma)$ smallest outlier fraction out of the $T$ mini-batches.
	Let $ \widehat{P}_\mathcal{U}^{(T)} $ denote the projection operator given by ORL-PCA, and $ P_\mathcal{U}^\star $ denotes the projection operator to the ground truth $ d $ dimensional subspace. Then, with a probability at least $ 1-\delta $,  we have,
	\begin{equation*}
	\|\widehat{P}_\mathcal{U}^{(T)}  -P_\mathcal{U}^\star   \|_F 	 \leq C_a \frac{\log(\log(T)/\delta) + 1}{T} 
 + c_1p\sqrt{\frac{d\log(1/\delta) }{b}} + c_2 \lambda(\gamma)\sqrt{dp}.
	\end{equation*}
	Here $ C_a,c_1,c_2 $ are  positive constants.
\end{theorem}

\subsection{Online Robust Linear Regression}
We then showcase another example of the application of ORL\textemdash online robust regression. As aforementioned, the target of linear regression is to estimate the parameter $ \theta^\star $ of linear regression model $ y_i=\langle \theta^\star, \mathbf{x}_i \rangle + \varepsilon $ given the observation pairs $ \{\mathbf{x}_i,y_i\}_{i=1}^{n+n_1} $ where $ n_1 $ samples are corrupted. Here $ \varepsilon \in \mathcal{N}(0,\sigma_e) $ is additive noise. Similar to ORL-PCA, we use the robustified thresholding (RoTR) regression proposed in Algorithm \ref{alg:rotr} (ref. \cite{chen2013robust}) as the robust learning procedure for parameter estimation within each mini-batch. 
\begin{algorithm}[h]
	\caption{Base Robust Regression - RoTR}
	\label{alg:rotr}
	\begin{algorithmic}
		\STATE {\bfseries Input:} Covariate matrix $X = [\mathbf{x}_1,\ldots,\mathbf{x}_{n+n_1}] \in \mathbb{R}^{p\times (n+n_1)}$ and response $\mathbf{y}\in\mathbb{R}^{n+n_1}$, outlier fraction $\lambda$ (set as $ 0.5 $ if unknown).
		\STATE For $j=1,\ldots,p$, compute $\widehat{\theta}(j)=\langle \mathbf{y}, X_j \rangle_{\lambda(n+n_1)}$.
		\STATE {\bfseries Output:} $\widehat{\theta}$.
	\end{algorithmic}
\end{algorithm}

Due to the blessing of online robust learning framework, ORL-LR has the following performance guarantee. 

\begin{theorem}
	\label{theo:rol_sr}
	Adopt the notations in Theorem \ref{theo:rol_pca}.
	Suppose  the authentic samples have the sub-Gaussian distribution as in \eqref{eqn:subgaussian}  with noise level $ \sigma_e $, are divided into $ T $ sequential mini-batches. Let $ \widehat{\theta}_T $ be the output of ORL-LR and $ \theta^\star $ be the ground truth. Then, with probability at least $ 1-\delta $,  the following holds:
	\begin{equation*}
	\|\widehat{\theta}_T-\theta^\star\|_2  \leq   C_a \frac{\log(\log(T)/\delta) + 1}{T} 
 + C_\gamma \|\theta^\star\|_2\sqrt{1+\frac{\sigma_e^2}{\|\theta^\star\|_2^2}}\Bigg(\sqrt{\frac{p\log (1/\delta)}{b}}  + \lambda(\gamma)\sqrt{p}\log (\frac{1}{\delta})\Bigg).
	\end{equation*}	
\end{theorem}

\subsection{Distributed Robust Learning}
\label{sec:drl}
Following the spirit of ORL, we can also develop a distributed robust learning (DRL) approach. Suppose in a distributed computing platform, $k$ machines are usable for parallel computation. Then for processing a large scale dataset, one can evenly distribute them onto the $k$ machines and run robust learning procedure $\text{RL}(\cdot, \nu)$ in parallel. Each machine provides  an independent estimate $\theta_i$ for the parameter of interest $\theta^\star$. Aggregating these estimates via geometric median (ref. Eqn.~\eqref{eqn:median}) can provide additional robustness to the inaccuracy, breakdown and communication error for a fraction of machines in the computing cluster, as stated in Lemma \ref{lemma:median}. Of particular interest, DRL can provide much stronger robustness than the commonly used averaging over the $k$ estimates, as average or mean is notoriously fragile to corruption. Even a single corrupted estimate out of the  $k$ estimates can make the final estimate arbitrarily bad.

\section{Simulations}
In this section, we investigate  robustness of the  ORL approach by evaluating the ORL-PCA and ORL-LR algorithms  and comparing them with their centralized and non-robust counterparts. We also perform similar investigation on DRL (ref. Section \ref{sec:drl}) considering robustness is also critical for distributed learning  in practice.
In the simulations, we report the results with the outlier fraction which is computed as $ \lambda = n_1/(n+n_1) $.



\begin{figure*}[t]
	\centering

	\subfigure[Online PCA]{
		\label{subfig:pca_non_uniform}
		\includegraphics[width=0.22\textwidth]{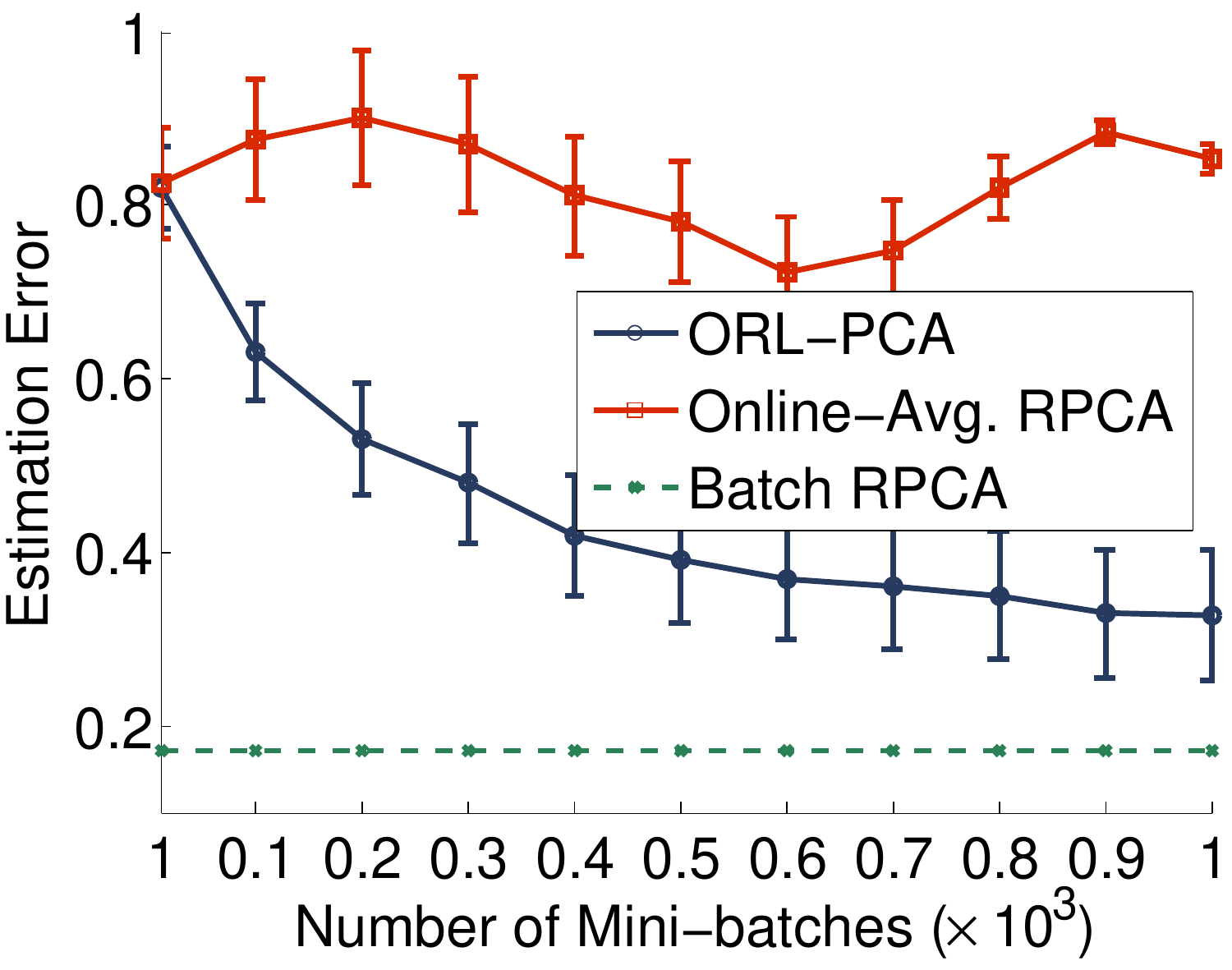}
	}	
	\subfigure[Online LR]{
		\label{subfig:online_lr}
		\includegraphics[width=0.22\textwidth]{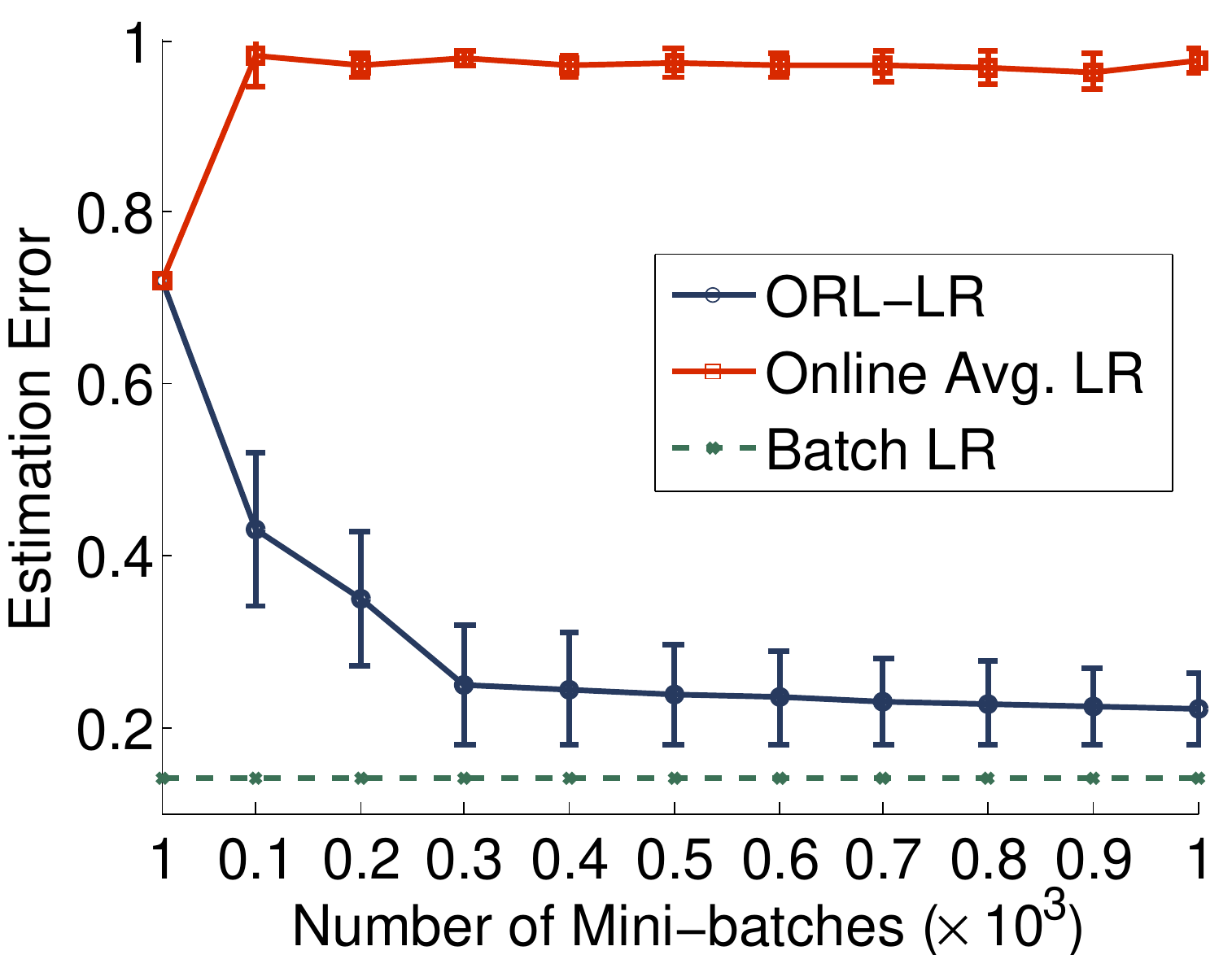}
	}
	\subfigure[Distributed PCA]{
	\label{fig:drpca}
	\includegraphics[width=0.22\textwidth]{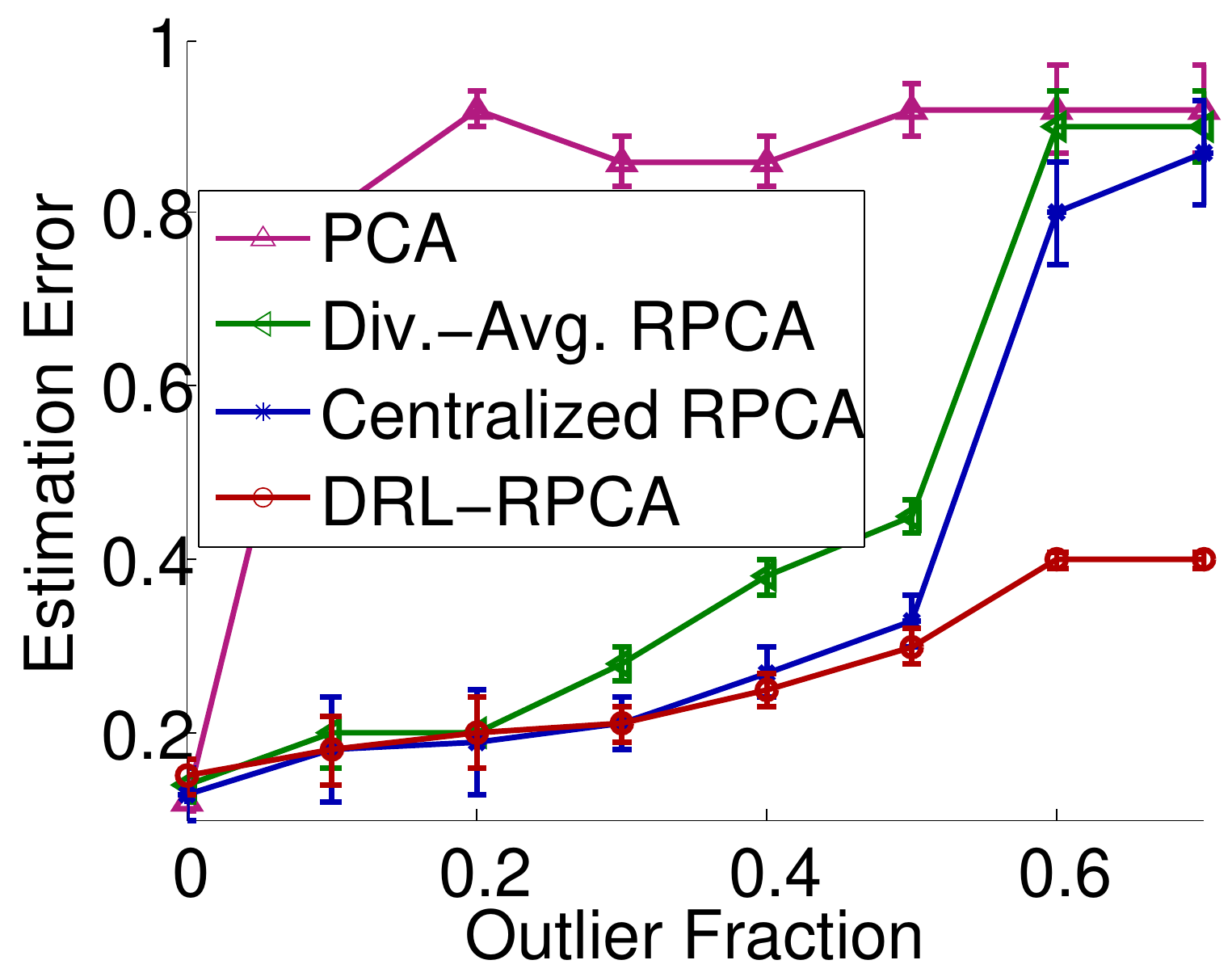}
}
\subfigure[Distributed LR]{
	\label{fig:drlr}
	\includegraphics[width=0.22\textwidth]{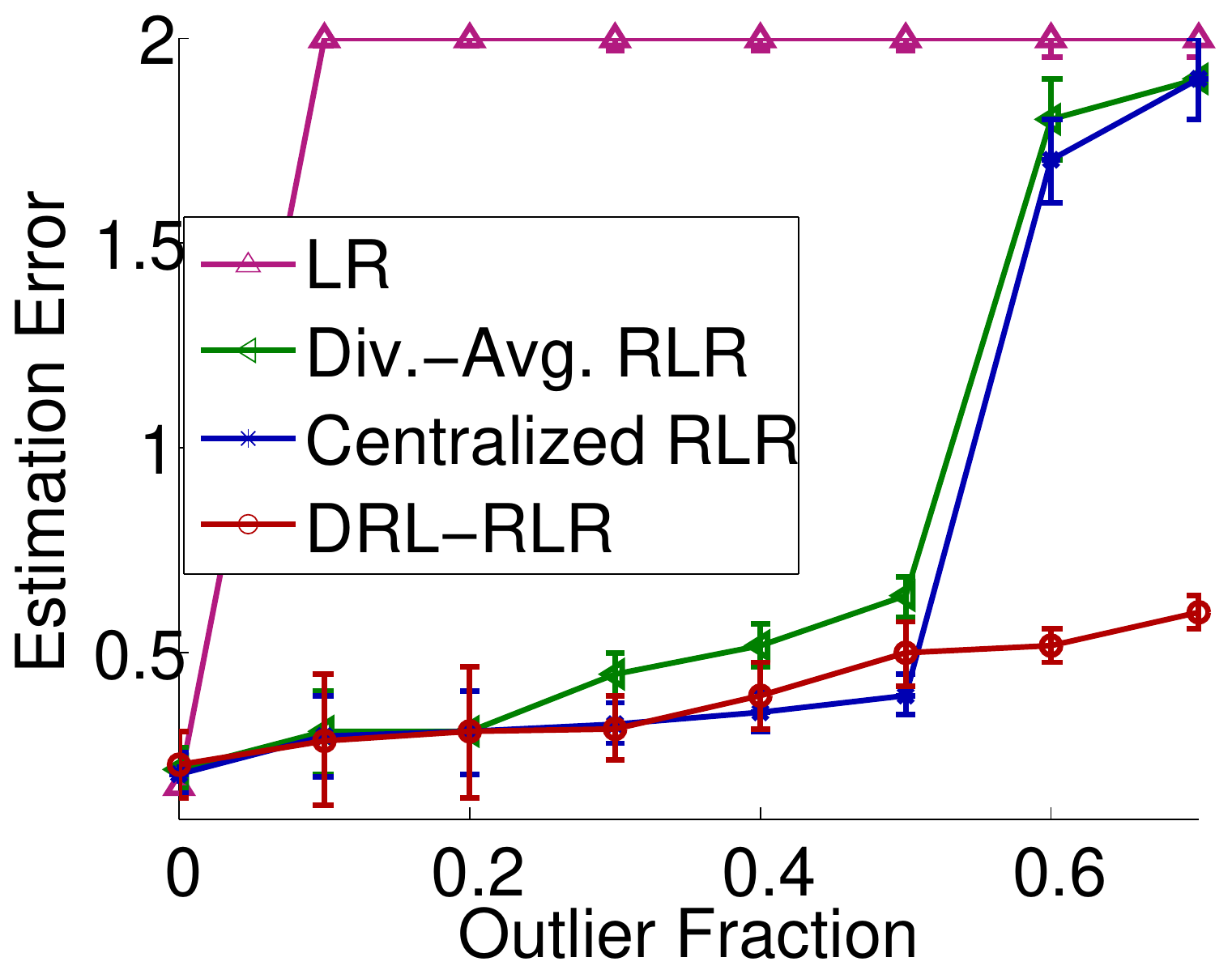}
}	
	\caption{Simulation comparison between online (in (a), (b)) and distributed as well as centralized (in (c), (d))  algorithms, along with  standard non-robust ones,  for  PCA and LR problems. Both problems have the following setting: noise level $\sigma_e = 1$, outlier magnitude $\sigma_o=10$, sample dimension $p=100$, sample size $N=1\times 10^{6}$, \# computation nodes $k=100$ (for distributed algorithms), and \# mini-batches $T=100$ (for online algorithms) . For PCA, intrinsic dimension $d=5$. (Best viewed in color.)
	}
	\label{fig:result}
	\vspace{-3mm}
\end{figure*}

\textbf{Data generation} In  simulations of the PCA problems, samples are generated according to $\mathbf{x}_i = \theta^\star\mathbf{z}_i + \varepsilon$. Here the signal $\mathbf{z}_i \in \mathbb{R}^d$ is sampled from the normal distribution: $\mathbf{z}_i \sim \mathcal{N}(0,I_d)$. The noise $\varepsilon \in \mathbb{R}^p$ is sampled as:  $\varepsilon \sim \mathcal{N}(0,\sigma_e I_p)$. The underlying matrix $\theta^\star \in \mathbb{R}^{p\times d}$ is randomly generated whose columns are then orthogonalized. The entries of outliers $\mathbf{x}_o \in \mathbb{R}^p$ are i.i.d.\ random variables from uniform distribution $[-\sigma_o,\sigma_o]$.
We use the distance between two projection matrices to measure the subspace estimation error for PCA: $\|\widehat{P}_\mathcal{U}-P_\mathcal{U}^\star\|_F/\|P_\mathcal{U}^\star\|_F$. Here $\widehat{P}_\mathcal{U}$ is the output estimates and $P_{\mathcal{U}}^\star=\theta^\star {\theta^\star}^\top$ is the ground truth.

In  simulations of the LR problems, samples $(\mathbf{x}_i,y_i)$ are generated according to $y_i = {\theta^\star}^\top \mathbf{x}_i + \varepsilon$. Here the model parameter $\theta^\star$ is randomly sampled from $\mathcal{N}(0,I_p)$ , and $\mathbf{x}_i \in \mathbb{R}^p$ is also sampled from normal distribution: $\mathbf{x}_i \in \mathcal{N}(0,I_p)$. The noise $\varepsilon \in \mathbb{R}$ is sampled as:  $\varepsilon \sim \mathcal{N}(0,\sigma_e)$. The entries of outlier $\mathbf{x}_o$ are also i.i.d.\ randomly sampled from uniform distribution $[-\sigma_o,\sigma_o]$. The response of outlier is generated by $y_o = -{\theta^\star}^\top \mathbf{x}_o + v$. We use $\|\theta^\star-\widehat{\theta}\|_2/\|\theta^\star\|_2$ to measure the error. Here $\widehat{\theta}$ is the output estimate.

\textbf{Online Setting} Results shown in Figure \ref{subfig:pca_non_uniform} give following observations. First, ORL-PCA converges to comparable performance with batch RC-PCA with accesses to  the entire data set. This demonstrates the rapid convergence of ORL-PCA. It is worth noting that ORL-PCA saves considerable memory cost than batch RC-PCA ($ 8 $ Mb \emph{vs.}\ $ 8\times 10^3 $Mb) and computation time ($ 212 $ seconds vs.\ $ \sim 27 $ Hours) since ORL-PCA performs SVD on much smaller data. Secondly, ORL-PCA offers much stronger robustness than naively averaged aggregation when outlier order is adversarial to corrupt a fraction of mini-batches. As shown in Figure \ref{subfig:pca_non_uniform}, when some batches have overwhelming outliers (outlier fraction $ \lambda_i \geq 0.5 $), base RC-PCA fails on these batches and outputs completely corrupted estimations. The corruption of mini-batches also fails online averaging RPCA. In contrast, ORL-PCA still offers correct estimation, even when a fraction of  $ 40\% $ of estimates from mini batches are corrupted.
We also report the results of ORL-LR and comparison with online-averaging baselines in Figure \ref{subfig:online_lr}. Similar to ORL-PCA, one observes that ORL-LR offers outperforming robustness to the sample outliers and batch corruptions, in contrast to the naive averaging algorithm.


\textbf{Distributed setting} All the simulations are implemented on a PC with $2.83$GHz Quad CPU and $8$GB RAM. It takes centralized RPCA around $60$ seconds to handle $1\times 10^6$ samples with dimensionality of $100$. In contrast, distributed RPCA only costs $0.6$ seconds by using $k=100$ parallel procedures.
The communication cost here is negligible since only eigenvector matrices of small sizes are communicated. For RLR simulations, we also observe about efficiency  enhancement. 

As for the performance, from Fig.~\ref{fig:drpca}, we observe that when $\lambda \leq 0.5$, DRL-RPCA, RPCA with division-averaging (Div.-Avg.\ RPCA) and centralized RPCA (\emph{i.e.}, the RC-PCA) achieve similar performances, which are much better than non-robust standard PCA. When $ \lambda = 0 $, \emph{i.e.}, when there are no outliers, the performances of DRL-RPCA and Div.-Avg.\ RPCA are slightly worse than standard PCA as the  quality of each mini-batch estimate deteriorates  due to the smaller sample size. However, distributed algorithms of course offer significant higher efficiency. Similar observations also hold for LR simulations from Fig.~\ref{fig:drlr}. Actually, standard PCA and LR begin to break down when $\lambda =0.1$. These results demonstrate that DRL preserves the robustness of centralized algorithms well. 

When outlier fraction $\lambda$ increases to $0.6$, centralized  (blue lines) and division-averaging algorithms (green lines) break down sharply, as the outliers outnumber their maximal breakdown point of $0.5$. In contrast, DRL-RPCA and DRL-RLR still present strong robustness and perform much better, which demonstrate that the DRL framework is indeed robust to computing nodes breaking down, and even enhances the robustness of the base robust learning methods under favorable outlier distributions across the machines.

\textbf{Comparison with Averaging}
Taking the average instead of the geometric median is a natural alternative to DRL. Here we provide more simulations for the  RPCA problem  to compare these two different aggregation strategies in the presence of different errors on the computing nodes.

In distributed computation of learning problems, besides outliers, significant deterioration of the performance may result from unreliabilities, such as  latency of some machines or  communication errors. For instance, it is not uncommon that machines solve their own sub-problem at different speed, and sometimes users may require to stop the learning before all the machines output the final results. In this case, results from the slow machines are possibly not accurate enough and may hurt the quality of the aggregated solution. Similarly, communication errors may also damage the overall performance. We simulate the machine latency by stopping the algorithms once over half of the machines finish their computation. To simulate communication error, we randomly sample $ k/10 $ estimations and flip the sign of $ 30\% $ of the elements in these estimations. The estimation errors of the solution aggregated by averaging and DRL are given in Table \ref{tab:avg-compare}. Clearly, DRL offers stronger resilience to unreliability of the computing nodes.

\begin{table}
	\caption{Comparisons of the estimation error for PCA between Division-Averaging (Div.-Avg.) and DRL, with machine latency and communication errors. Under the same parameter setting as Figure~\ref{fig:drpca}. Outlier fraction $ \lambda = 0.4 $. The average and std of the error from $ 10 $ repetitions are reported.}
	\label{tab:avg-compare}
	\centering
	\begin{tabular}{c|c|c}
		Unreliability Type  & DRL & Div.-Avg. \\
		\hline
		Latency & $ 0.26 \pm 0.01 $ & $ 0.42 \pm 0.01 $  \\
		\hline
		Commu. Error &$ 0.31 \pm 0.03 $  & $ 0.78 \pm 0.02  $   \\
	\end{tabular}
	\vspace{-4mm}
\end{table}

\textbf{Real large-scale data} We further apply the ORL-LR   for an image tag prediction problem on a large-scale image set, \emph{i.e.}, the Flickr-10M image set\footnote{\url{http://webscope.sandbox.yahoo.com/catalog.php?datatype=i&did=67}}, which contains  $ 1 \times 10^8 $ images with \emph{noisy} users contributed tags. We employ robust linear regression to predict $ 200 $ semantic tags for each image, which is described by a  $ 4{,}096$-dimensional deep CNN  feature (output of the fc6 layer) \cite{jia2014caffe}. Performing such a large scale regression task is  impossible for a single PC (with a  $ 16 $GB memory), as only storing the features costs nearly $ 50 $GB memory. Therefore, we solve this problem via the proposed online and distributed learning algorithms. We randomly sample from the entire dataset  a training set of $ 0.5\times 10^8 $ images and a test set of $ 0.1\times 10^8 $ images.

\begin{table}
	\caption{Tag prediction accuracy comparisons among online LR algorithms  on the Flickr-10M dataset and their computation time (in secs.).}
	\label{tab:flickr-acc_orl}
	\centering
	\begin{tabular}{c|c|c|c}
		Alg. &  ORL-LR & Online Avg. LR & Stoc. LR\\
		\hline
		Accuracy  & $ \mathbf{0.36 \pm 0.02} $ & $ 0.27 \pm 0.01 $ & $ 0.23 \pm 0.01 $  \\
		\hline
		Time (secs.)  &$ 4{,}320 \pm 10 $  & $ 3{,}960 \pm 9 $  & $ 2{,}880 \pm 12 $ \\
	\end{tabular}
\end{table}


We  perform experiments with the online learning setting, and compare the performance of the proposed ORL-LR with the Online Averaging LR. We also implement a non-robust baseline~--~stochastic gradient descent to solve the LR problem. The size of min-batch is fixed as $ 5\times 10^5 $ images. From the results in Table \ref{tab:flickr-acc_orl}, one can observe that ORL-LR achieves significantly higher accuracy than  non-robust baseline algorithms, with a margin of more than $ 9\% $.

\section{Proofs}
\subsection{Technical Lemmas}
\begin{lemma}[Hoeffding's Inequality]
	\label{lemma:Hoeffding}
	Let $ X_1, \ldots, X_n $ be independent random variables taking values in $ [0,1] $. Let $ \bar{X}_n=\frac{1}{n}\sum_{i=1}^n X_i  $ and $ \mu = \mathbb{E}\bar{X}_n $. Then for $ 0<t<1-\mu $,
	\begin{equation*}
	\mathbb{P}\left(\bar{X}_n - \mu \geq t \right) \leq  \exp \left\{-2nt^2\right\}.
		\end{equation*}
		\end{lemma}
		
		\begin{lemma}[A coupling result \cite{lerasle2011robust}]
			\label{lemma:coupling}
			Let $ Y_{1:N} $ be independent random variables, let $ x $ be a real number and let $ A = \operatorname{Card}\left\{i=1,\ldots,N \text{ s.t. } Y_i >x\right\} $. Let $ p\in (0,1] $ such that, for all $ i=1,\ldots,N $, $ p \geq \mathbb{P}\{Y_i > x\} $ and let $ B $ be a random variable with a binomial law $ \operatorname{Bin}(N,p) $. There exists a coupling $ \tilde{C}=(\tilde{A},\tilde{B}) $ such that $ \tilde{A} $ has the same distribution as $ A $, $ \tilde{B} $ has the same distribution as $ B $ and such that $ \tilde{A}\leq \tilde{B} $. In particular, for all $ y>0 $,			$\mathbb{P}\left\{A>y\right\} \leq \mathbb{P}\left\{B>y\right\}$.
			\end{lemma}
			The following lemma demonstrates that aggregating estimates via their geometric median can enhance the confidence significantly. 
			\begin{lemma}
				\label{lemma:median_aggregation}
				Given $ k $ independent estimates of $ \theta^\star $ satisfying $ \mathbb{P}(\|{\theta}_i - \theta^\star \| \leq  R) \geq p^* >1/2 $, for all $ i=1,\ldots, k $. Let $ \widehat{\theta}= \operatorname{median}({\theta}_1,\ldots, {\theta}_k) $, then we have,
				\begin{equation*}
				\mathbb{P}\left(\|\widehat{\theta} - \theta^\star \| < C_\gamma R \right) \geq 1-\exp\left\{ -2k(\gamma-1+p^*)^2  \right\}.
				\end{equation*}
				Here $\gamma$ and $C_\gamma$ are defined in Lemma \ref{lemma:median}.
				\end{lemma}		
				\begin{proof}
					 According to Lemma \ref{lemma:median}, we have
					\begin{equation*}
					\mathbb{P}\left(\|\widehat{\theta} - \theta^\star \| \geq C_\gamma R \right) \leq \mathbb{P}\left(\sum_{j=1}^k \mathds{1}\left(\|{\theta}_j - \theta^\star\| \geq R \right) \geq \gamma k\right).
					\end{equation*}	
					Let $ Z_j =  \mathds{1}\left\{\|{\theta}_j - \theta^\star \| > R \right \} \sim \operatorname{Ber}(1-p^*)$, and let $ W=\sum_{j=1}^k Z_j$ have Binomial distribution $ W \sim \operatorname{Bin}(k,1-p^*) $, then
					\begin{equation*}
					\mathbb{P}\left(\sum_{j=1}^k \mathds{1}\left\{\|{\theta}_j - \theta^\star \| > R \right \} > \gamma k \right) \leq \mathbb{P}(W > \gamma k),
					\end{equation*}
					according to Lemma \ref{lemma:coupling}. Applying the Hoeffding's inequality in Lemma \ref{lemma:Hoeffding} (with $ \mu = 1-p^* $, $ t = \gamma - 1 + p^* $, $\gamma < 1/2$ and $ p^* > 1/2 $) gives
					\begin{equation*}
							\mathbb{P} \left( \|\widehat{\theta}- \theta^\star\| > C_\gamma R \right) \leq  \mathbb{P}\left(\sum_{j=1}^k \mathds{1}\left\{\|{\theta}_j - \theta^\star \| > R \right \} > \gamma k \right) 
						\quad \leq  \mathbb{P}\left(W  > \gamma k  \right) \leq \exp\left\{-2k (\gamma - 1+p^*)^2  \right\}
						\end{equation*}
						\end{proof}
						
\subsection{Proofs of Main Results on ORL}
\subsubsection{Conditions}
We suppose from now on   following conditions hold.

\begin{condition}
	\label{cond1}
	Assume that $ \theta^\star \in \Theta $ is the parameter of interest. Let $ \theta_1,\ldots,\theta_T \in \Theta$ be a collection ofindependent estimates of $ \theta^\star $, which are not concentrated on a straight line: for all $ v \in \Theta $, there is $ w \in \Theta $ such that $ \langle v,w \rangle =0$ and $\frac{1}{T} \sum_{i=1}^T \|\langle w, \theta_i - \bar{\theta}\rangle \| > 0$ with $ \bar{\theta} = \frac{1}{T}\sum_{i=1}^T \theta_i $. 
\end{condition}
As noted in~\cite{cardot2013efficient},  Condition \ref{cond1} ensures that  geometric median $ \widehat{\theta} $ of the $ T $ estimates is uniquely defined.

\begin{condition}
	\label{cond2}
	The distribution of the independent estimates of $ \theta^\star $ is a mixing of two ``nice'' distributions: $ \mu_{\theta^\star} =  \mu_c + \mu_d $. Here $ \mu_c $ is not strongly concentrated around single points: if $ {B}(0,a) $ is the ball $ \{u \in \Theta, \|u\| \leq a \} $, and $ Y $ is a random variable with distribution $ \mu_c $, then for any constant $ a > 0 $,
	\begin{equation*}
	\exists C_a \in [0,\infty), \forall u \in {B}(0,a), \mathbb{E}_Y \left[\|Y-u\|^{-1}\right] \leq C_a.
	\end{equation*}
	In addition, $ \mu_d $ is a discrete measure, $ \mu_d = \sum_i  \delta_{u_i} $. Here $ \delta_{u_i} $ is a Dirac measure at point $ u_i $. We denote by $ D $ the support of $ \mu_d $ and assume that the median $ \widehat{\theta} \notin D $.
\end{condition}

Conditions \ref{cond1} and \ref{cond2} are only technical conditions to avoid pathologies in the convergence analysis for Algorithm \ref{alg:ROL}. In practical implementations, we can simply set the  sub-gradient of $ G(u) $ at $ u^\prime $ as zero (a valid sub-gradient as proved in \cite{cardot2013efficient}) when  $ u^\prime \in D $.

\subsubsection{Convergence Rate of Geometric Median Filtering}

Given the definition of geometric median in~\eqref{eqn:median}, we can define following population geometric median loss function, $ G: \Theta \rightarrow \mathbb{R} $, that we want to minimize to compute the geometric median:
\begin{equation}
\label{eqn:func_G}
G(u) \triangleq \mathbb{E}\left[\|\Theta-u\| - \|\Theta\| \right].
\end{equation}

In this subsection, we first show that the geometric median function in \eqref{eqn:func_G} is indeed strongly convex under Conditions \ref{cond1} and \ref{cond2}. Thus the SGD optimization is able to provide solutions with a convergence rate of $ O(\log\log(T)/T) $ to the true geometric median $ \widehat{\theta} $, given $ T $ independent estimates.

\begin{definition}[$ \beta $-strongly convex function \cite{nedic2003convex}]
	\label{def:strong-convex}
	A function $ G $ is $ \beta $-strongly convex, if for all $ u_1,u_2 \in \Theta $ and any sub-gradient $ g(u) $ of $ G $ at $ u $, we have
	\begin{equation*}
	\langle g(u_2)-g(u_1), u_2 - u_1 \rangle \geq \beta \|u_2-u_1\|^2.
	\end{equation*}
\end{definition}
The following theorem establishes the strong convexity of the geometric median function in \eqref{eqn:func_G}.
\begin{theorem}
	\label{theo:convexity}
	Let $ g(u) $ be the sub-gradient of $ G(u) $ at $ u $.
	Under Conditions \ref{cond1} and \ref{cond2}, there is a strictly positive constant $ c_a >0 $, such that:
	\begin{equation*}
	\forall u_1,u_2 \in {B}(0,a), \langle g(u_2)-g(u_1) , u_2-u_1\rangle  \geq c_a \|u_2-u_1\|^2,
	\end{equation*}
	and thus $ G(u) $ is $ c_a $-strongly convex.
\end{theorem}
The  proof can be derived  from the proof for the Proposition 2.1 in~\cite{cardot2013efficient} straightforwardly and we omit details here.

Given the strong convexity property of geometric median function $ G(u) $, we  can apply the convergence argument of  SGD  for strongly convex functions (\emph{e.g.}, Proposition 1 in \cite{rakhlin2011making}), and obtain the following convergence rate for online geometric median filtering.
\begin{theorem}
	\label{theo:converge_rate}
	Assume Conditions \ref{cond1} and \ref{cond2} hold, and $ \|\theta_i\| \leq K, \forall i=1,\ldots,T $.  Then  $ \|\widehat{g}_t\|^2 \leq K C_a $ with probability $1$.   Assume $ T \geq 4 $. Let $ \delta \in (0,1/e) $. Pick $ \eta_t= {1}/{t c_a } $ in Algorithm \ref{alg:ROL} and let $\widehat{\theta}_t$ denote the output at time step $t$. Furthermore,  let $ \widehat{\theta} $ be the geometric median of $ \{\theta_i\}_{i=1}^T $. Then for any $ t \leq T $,
	\begin{equation*}
	\|\widehat{\theta}_t - \widehat{\theta}\| \leq C_a^\prime \frac{( \log(\log(t)/\delta)+1)}{ t},
	\end{equation*}
	with probability at least $ 1-\delta $. Here $ C_a^\prime = KC_a/c_a^2 $.
\end{theorem}
The bound on the gradient $ \|\widehat{g}_t\|^2 \leq K C_a  $ is from the definition of the gradient in~\eqref{eqn:sgdmed}, Condition \ref{cond2} and the assumption that all the estimates are bounded.

\subsubsection{Proofs of Proposition \ref{prop:rol}}
From now on, we slightly abuse the notation and use $\widetilde{\theta}$ to denote the geometric median of a collection of estimates.
\begin{proof}
	Proposition \ref{prop:rol} can be derived by following triangle inequality:
	\begin{equation*}
	\|\widehat{\theta}_t - \theta^\star\| \leq \|\widehat{\theta}_t - \widetilde{\theta}\| + \|\widetilde{\theta} - \theta^\star\|,
	\end{equation*}
	where $\widetilde{\theta}$ denotes the ``true'' geometric median of estimates $\{\theta_i\}_{i=1}^t$.
	We now proceed to bound the above two terms separately. Based on Theorem \ref{theo:converge_rate}, we have
	\begin{equation*}
	\|\widehat{\theta}_t - \widetilde{\theta}\| \leq C_a^\prime \frac{( \log(\log(t)/\delta)+1)}{ t},
	\end{equation*}
	with a probability at least $ 1-\delta $. The second term can be bounded as follows by applying Lemma \ref{lemma:median_aggregation}:
	\begin{equation*}
	\mathbb{P}\left(\Norm{\widetilde{\theta}  - \theta^\star } \lesssim_{\delta, L}  \sqrt{\frac{1}{b}} + \lambda(\gamma) \sqrt{p} \right) \geq 1-\delta,
	\end{equation*}
	where $\lambda(\gamma) = \frac{\lambda_{(1-\gamma)}}{1-\lambda_{(1-\gamma)}}$ and $\lambda_{(1-\gamma)}$ denotes the $ \lfloor (1-\gamma) k \rfloor $ smallest outlier fraction in $ \{\lambda_1,\ldots,\lambda_k\} $ with $ \gamma \in [0,1/2) $. Combining these two bounds together gives:
	\begin{equation*}
	\|\widehat{\theta}_t - \theta^\star\| \lesssim_{\delta,L} C_a^\prime \frac{ \log(\log(t)/\delta)+1}{t} + C_\gamma  \sqrt{\frac{1}{b}} + C^\prime \lambda(\gamma)\sqrt{p}.
	\end{equation*}
	
\end{proof}

\subsection{Proofs of Application Examples}
Before proving the performance guarantee for ORL-PCA and ORL-LR, we provide robustness analysis for the base robust learning procedure\textemdash the RC-PCA and RoTR.
\subsubsection{Robustness Guarantee of RC-PCA}
\begin{theorem}
	\label{theo:rpca}
	Suppose in total $ N $ samples are provided with $n$  authentic samples and $n_1$ outliers. Let  $\lambda = n_1/N$.
	Assume the authentic samples follow sub-Guassian design with parameter $ L $.
Let $\Delta_d = \sigma_d - \sigma_{d+1}$, where $\sigma_d$ denotes the $d^{th}$ largest eigenvalue of ground-truth sample covariance matrix $C^\star$.
	Let $P_{\mathcal{U}}$ be the output $d$-dimensional subspace projector from RC-PCA. Then for a constant $c$,  we have with probability $1-\delta$,
	\begin{equation*}
	 \| P_{\mathcal{U}} - P_{\mathcal{U}}^\star\|_\infty \leq     \frac{2L}{\Delta_d} \left\{ \sqrt{\frac{4}{c} \log(\frac{4}{\delta})} \sqrt{\frac{p}{n}}  
	 +  \frac{\lambda}{1-\lambda } \log (\frac{2}{\delta}) \right\}.
	\end{equation*}
\end{theorem}
\begin{proof}
	According to the proof of Theorem 4 in \cite{chen2013robust} and deviation bound on the empirical covariance matrix estimation in \cite{vershynin2012close}, when the authentic samples are from sub-Gaussian distribution with parameter $ L $,  we have, for the covariance matrix constructed in Algorithm~\ref{alg:rpca},
	\begin{equation*}
	\|\widehat{C}-C^\star\|_\infty \leq L \sqrt{\frac{4}{c} \log(\frac{4}{\delta})}\sqrt{\frac{p}{n} } +  \frac{n_1}{n}L \log (\frac{2}{\delta}) 
	\end{equation*}
	with a probability at least $ 1-\delta $. Here $ c $ is a constant, $n$ is the number of authentic samples and $n_1$ is the number of outliers.
	
	Let $\Delta_d = \sigma_d - \sigma_{d+1}$ be the eigenvalue gap, where $\sigma_d$ denotes the $d$-th largest eigenvalue of $C^\star$.
	Then, applying the Davis-Kahan perturbation theorem \cite{davis1970rotation}, we have, whenever $\|\widehat{C}-C^\star\|_\infty \leq \frac{1}{4}\Delta_d$, $\|P_\mathcal{U}-P_\mathcal{U}^\star\|_\infty \leq {2\|\widehat{C}-C^\star\|_\infty}/{\Delta_d}$. Thus,
	\begin{equation*}
	\label{eqn:bounded_proj}
	 \|P_\mathcal{U}-P_\mathcal{U}^\star\|_\infty  \leq \frac{2L}{\Delta_d} \left\{ \sqrt{\frac{4}{c} \log(\frac{4}{\delta})}\sqrt{\frac{p}{n} } +  \frac{n_1}{n} \log (\frac{2}{\delta}) \right\},
	\end{equation*}
	with a probability at least $ 1-\delta $.
\end{proof}

\subsubsection{Proof of Theorem \ref{theo:rol_pca}}
\begin{proof}
	Theorem \ref{theo:rol_pca} can be derived directly from following triangle inequality:
	\begin{equation*}
	\|\widehat{P}_\mathcal{U}^{(T)} - {P}_{\mathcal{U}}^{\star} \|_F \leq \|\widehat{P}_\mathcal{U}^{(T)} - \widetilde{P}_\mathcal{U} \|_F + \|\widetilde{P}_\mathcal{U} - {P}_{\mathcal{U}}^{\star}\|_F,
	\end{equation*} 
	and we  bound the above two terms separately. The first term can be bounded by Theorem \ref{theo:converge_rate} as,
	\begin{equation*}
	\|\widehat{P}_\mathcal{U}^{(T)} - \widetilde{P}_\mathcal{U} \|_F \leq C_a^\prime \frac{( \log(\log(T)/\delta)+1)}{ T},
	\end{equation*}
	with a probability $ 1-\delta $. The second term can be bounded as in Theorem \ref{theo:rpca} that with a probability $1-\delta$,
	\begin{equation*}
	 \|\widetilde{P}_\mathcal{U} - {P}_{\mathcal{U}}^{\star}\|_F \leq c_1p\sqrt{\frac{d\log(1/\delta) }{N}} + c_2 \lambda(\gamma)\sqrt{dp}.
	\end{equation*}	
	Combining the above two bounds (with union bound) proves the theorem.
\end{proof}

\subsubsection{Proof of Theorem \ref{theo:rol_sr}}
Before proving Theorem~\ref{theo:rol_sr}, we first show the following performance guarantee for RoTR algorithm from \cite{chen2013robust}.
The estimation error of the RoTR is bounded as in Lemma~\ref{lemma:RoTR}.
\begin{lemma}[Performance of RoTR~\cite{chen2013robust}]
	\label{lemma:RoTR}
	Suppose the samples $ \mathbf{x} $ are from sub-Gaussian design with $\Sigma_x = I_p$, with dimension $ p $ and noise level $ \sigma_e $, then with probability at least $1-\delta$, the output of RoTR satisfies the $\ell_2$ bound:
	\begin{align*}
	\left\|\widehat{\theta}-\theta^\star \right\|_2 \leq c\|\theta^\star\|_2\sqrt{1+\frac{\sigma_e^2}{\|\theta^\star\|_2^2}}& \left(\sqrt{\frac{p\log (1/\delta)}{n}} \right. \\
	& \left. + \frac{\lambda}{1-\lambda}\sqrt{p}\log (1/\delta)\right).
	\end{align*}
	Here $c$ is a constant independent of $p,n,\lambda$.
\end{lemma}


\begin{proof}
		Based on the results in the Lemma \ref{lemma:RoTR} and Lemma \ref{lemma:median_aggregation}, it is straightforward to get:
		\begin{equation*}
		\|\widetilde{\theta} -\theta^\star\|_2 \leq C'_\gamma \|\theta^\star\|_2\sqrt{1+\frac{\sigma_e^2}{\|\theta^\star\|_2^2}} 
		 \left(\sqrt{\frac{p\log (1/\delta)}{N}}  + \lambda(\gamma)\sqrt{p}\log (1/\delta)\right)
		\end{equation*}
		where $C'_\gamma = C_\gamma c$ with $c$ being the constant in Lemma~\ref{lemma:RoTR},  $ C_\gamma = (1-\gamma)\sqrt{\frac{1}{1-2\gamma}} $, and $\lambda(\gamma)=\lambda_{(1-\gamma)}/(\lambda_{(1-\gamma)}$ with $ \lambda_{(1-\gamma)}  $ being the $ \lfloor k (1-\gamma) \rfloor $ smallest outlier fraction in $ \{\lambda_1,\ldots,\lambda_k\} $.
	\end{proof}

	As proving Theorem \ref{theo:rol_pca}, Theorem \ref{theo:rol_sr} can be derived based on the results in Theorem \ref{theo:converge_rate}. For simplicity, we omit the details here.

\section{Conclusions}

We developed a generic  Online Robust Learning (ORL) approach with provable robustness guarantee  and we also demonstrate its application for  Distributed Robust Learning (DRL). The proposed approaches not only significantly enhance the time and memory efficiency of robust learning  but also preserve the robustness of the centralized learning procedures.  Moreover, when the outliers are not uniformly distributed, the proposed approaches are still robust to adversarial outliers distributions. We  provided two concrete examples, online and distributed robust principal component analysis and linear regression.

\end{document}